\documentclass[sigconf]{acmart}

\usepackage{color}
\usepackage{algorithm}
\usepackage{algpseudocode}
\usepackage{undertilde}

\usepackage{array}
\usepackage{eqparbox}
\usepackage{stfloats}
\usepackage{url}
\usepackage{color}
\usepackage{multirow}
\usepackage{amsmath,amsfonts}
\usepackage{subfigure}
\usepackage{graphicx}
\usepackage{tikz}

\newcommand{\pgftextcircled}[1]{
    \setbox0=\hbox{#1}%
    \dimen0\wd0%
    \divide\dimen0 by 2%
    \begin{tikzpicture}[baseline=(a.base)]%
        \useasboundingbox (-\the\dimen0,0pt) rectangle (\the\dimen0,1pt);
        \node[circle,draw,outer sep=0pt,inner sep=0.1ex] (a) {#1};
    \end{tikzpicture}
}

\newcommand{\pgftextcircledblk}[1]{
    \setbox0=\hbox{#1}%
    \dimen0\wd0%
    \divide\dimen0 by 2%
    \begin{tikzpicture}[baseline=(a.base)]%
        \useasboundingbox (-\the\dimen0,0pt) rectangle (\the\dimen0,1pt);
        \node[circle,draw,outer sep=0pt,inner sep=0.1ex,fill=blue] (a) {#1};
    \end{tikzpicture}
}

\newcommand{\red}[1]{
       \textcolor{red}{#1}
}

\def\sign{{\rm sign}}
\def\supp{{\rm supp}}

\def\t0{{t_0}}

\def\E{{\mathbb E}}

\def\R{{\mathbb R}}        
\def\Rn{{\R^{n}}}          

\def\t0{{t_0}}
\def\vs{{\mathbf s}}
\def\vY{{\mathbf Y}}
\def\vX{{\mathbf X}}
\def\vE{{\mathbf E}}
\def\vH{{\mathbf H}}

\def\vI{{\mathbf I}}
\def\vN{{\mathbf N}}

\def\Proj{{\rm Proj}}

\DeclareMathOperator*{\Min}{minimize}
      \newtheorem{prop}{Proposition}
   \newtheorem{remark}{Remark}

\begin{document}
\title{Exploring Outliers in Crowdsourced Ranking for QoE}

\copyrightyear{2017}
\acmYear{2017}
\setcopyright{acmcopyright}
\acmConference{MM '17}{October 23--27, 2017}{Mountain View, CA, USA}\acmPrice{15.00}\acmDOI{10.1145/3123266.3123267}
\acmISBN{978-1-4503-4906-2/17/10}

\fancyhead{}
\settopmatter{printacmref=false, printfolios=false}

\author{Qianqian Xu$^{1}$, Ming Yan$^{2}$, Chendi Huang$^{3}$, \\Jiechao Xiong$^{3,4}$, Qingming Huang$^{5,6,7}$, Yuan Yao$^{8}$$^\ast$}
\thanks{$^\ast$Corresponding author.}
\affiliation{%
  \institution{$^1$ State Key Laboratory of Information Security (SKLOIS), Institute of Information Engineering, CAS, Beijing 100093, China}
  \institution{$^2$ Department of Computational Mathematics, Michigan State University, East Lansing, MI 48824, USA }
  \institution{$^3$ BICMR-LMAM-LMEQF-LMP, School of Mathematical Sciences, Peking University, Beijing 100871, China}
   \institution{$^4$ Tencent AI Lab, Shenzhen 518057, China}
   \institution{$^5$ University of Chinese Academy of Sciences, Beijing, 100049, China}
  \institution{$^6$ Key Lab of Intell. Info. Process., Inst. of Comput. Tech., CAS, Beijing, 100190, China}
  \institution{$^7$ Key Lab of Big Data Mining and Knowledge Management, CAS, Beijing, 100190, China}
  \institution{$^8$ Department of Mathematics, Hong Kong University of Science and Technology, Hong Kong 100871}
}
\email{xuqianqian@iie.ac.cn,yanm@math.msu.edu,cdhuang@pku.edu.cn}
\email{jcxiong@tencent.com, qmhuang@ucas.ac.cn,yuany@ust.hk}


\begin{abstract}

 Outlier detection is a crucial part of robust evaluation for crowdsourceable assessment of Quality of Experience (QoE) and has attracted much attention in recent years. In this paper, we propose some simple and fast algorithms for outlier detection and robust QoE evaluation based on the nonconvex optimization principle. Several iterative procedures are designed with or without knowing the number of outliers in samples. Theoretical analysis is given to show that such procedures can reach statistically good estimates under mild conditions. Finally, experimental results with simulated and real-world crowdsourcing datasets show that the proposed algorithms could produce similar performance to Huber-LASSO approach in robust ranking, yet with nearly 8 or 90 times speed-up, without or with \emph{a prior} knowledge on the sparsity size of outliers, respectively. Therefore the proposed methodology provides us a set of helpful tools for robust QoE evaluation with crowdsourcing data.
\end{abstract}


%
%

%
%



\maketitle

\section{Introduction}
\label{Intro}
In recent years, the Quality of Experience (QoE)~\cite{Hossfeld12-QoE,Wu13crowd} has become a major research theme within the multimedia community. QoE measures  a user's subjective expectation, feeling, perception, and satisfaction with respect to multimedia content. Measuring and ensuring good QoE of multimedia content is highly subjective in nature.

A variety of approaches can be employed to conduct subjective tests, among which Mean Opinion Score (MOS)~\cite{MOS} and paired comparison are the two most popular ones. In the MOS test, individuals are asked to specify a rating from ``Bad" to ``Excellent" (e.g., Bad-1, Poor-2, Fair-3, Good-4, and Excellent-5) to grade the quality of a stimulus; while in paired comparison approach, raters are only asked to make intuitive comparative judgments instead of mapping their perception on a categorical or numerical scale. Among these there may be tradeoffs in the amount of information the preference label contains and the bias associated with obtaining the label. For example, while a graded relevance judgment on a five-point scale may contain more information than a binary judgment, raters may also make more errors due to the complexity of assigning finer-grained judgments. In~\cite{MM09}, it shows that MOS may suffer from three fundamental problems: (i) it is unable to concretely define the concept of scale; (ii) the interpretations of the scales among raters are highly different; (iii) it is difficult to verify whether a rater gives false ratings either intentionally or carelessly. Therefore, the paired comparison method is currently gaining growing attention. It not only promises assessments that are easier and faster to obtain with less demanding task for raters, but also yields more reliable data with less personal scale bias in practice. However, a shortcoming of paired comparison is that it has more expensive sampling complexity than the MOS test, since the number of pairs grows quadratically with the number of items to be ranked.

To tackle the cost problem, with the growth of crowdsourcing platforms such as \href{https://www.mturk.com}{MTurk}, \href{http://www.innocentive.com/}{InnoCentive}, \href{http://crowdflower.com/}{CrowdFlower},~\href{http://www.crowdrank.net/}{CrowdRank}, and \href{http://www.allourideas.org/}{AllOurIdeas}, researchers who wish to seek help from the Internet crowd can post their task requests on websites for QoE evaluation ~\cite{MM09,tmm12,Hossfeld2014,conf2012-441,Keimel_etal_QoMEX2012_CrowdSourcing_Preprint,Wu13crowd}.
Methods for rating/ranking via pairwise comparisons in QoE evaluation in crowdsourcing scenarios must address a number of inherent difficulties including: (i) incomplete and imbalanced data; (ii) streaming and online data; (iii) outliers. To meet the first challenge, the work in~\cite{added,MM11,tmm12} propose randomized paired comparison methods which accommodate incomplete and imbalanced data. A general framework named~\emph{HodgeRank on Random Graphs} (HRRG) not only deals with incomplete and imbalanced data collected from crowdsourcing studies but also derives the constraints on sampling complexity  that the random selection must adhere to in crowdsourcing experiment. Furthermore, a recent extension of HRRG is introduced in~\cite{MM12,TMM13} to deal with streaming and online data in the crowdsourcing scenario, providing the possibility of making assessment procedure significantly faster without deteriorating the accuracy.

The third challenge of crowdsourcing QoE evaluations is because not every Internet rater is trustworthy.
In other words, due to the lack of supervision when raters perform experiments in crowdsourcing,
they may provide erroneous responses perfunctorily, carelessly, or dishonestly~\cite{MM09}.
Such random decisions are useless and may deviate significantly from other raters' decisions.
So outliers have to be identified and removed in order to achieve a robust QoE evaluation.
Many methods have been developed for outlier detection, such as $M$-estimator ~\cite{Huber81}, Least Median of Squares (LMS)~\cite{rousseeuw1984least}, S-estimators~\cite{rousseeuw1984robust}, Least Trimmed Squares (LTS)~\cite{leroy1987robust}, and Thresholding
based Iterative Procedure for Outlier Detection ($\Theta$-IPOD) \cite{SheOwe11} etc.  Besides, there are also distribution-based~\cite{barnett1994outliers}, depth-based~\cite{johnson1998fast}, distance-based~\cite{knorr1999finding,knorr2000distance},
density-based~\cite{breunig2000lof}, and clustering-based~\cite{jain1999data} methods for outlier detection.
The authors of~\cite{MM09} proposed Transitivity Satisfaction Rate (TSR), which checks all
the intransitive triangles, e.g., $A \succ B \succ C\succ A$, to identify and discard noisy
data provided by unreliable raters in QoE. However, TSR can only be applied to complete and balanced
paired comparison data. When the paired data is incomplete and imbalanced, {i.e.}, having missing edges,
the question of how to detect the noisy pairs remains open. The work in~\cite{MM13} attacks this problem
and formulates the outlier detection as a LASSO problem based on sparse approximations of the cyclic ranking
projection of paired comparison data in Hodge decomposition. Regularization paths of the LASSO problem could
 provide an order on samples tending to be outliers. However, the solution of the LASSO problem is biased.
 Solving the LASSO path is too slow and the problem has to be solved for many times for model selection
  via cross-validation.

%
%
%

In this paper, we propose simple and fast algorithms based on nonconvex optimization for outlier detection and robust ranking in QoE evaluation. The contributions of this paper are as follows:

1. We propose 3 iterative procedures solving some nonconvex optimization problems arising from outlier detection with or without knowing the number of outliers in samples.

2. Theoretical analysis shows that such procedures can reach statistically good estimates under mild conditions.

3. Experiments with simulated and crowdsourcing real-world data show that our algorithms work effectively in practice.

\section{METHODOLOGY} \label{sec:Outlier_Detection}
 In this section, we propose some simple iterative algorithms for outlier detection by solving some nonconvex optimization problems. These algorithms are based on either a prior knowledge on the number of outliers or adaptive estimation of the outlier sparsity size. Specifically, we propose iHT and iLTS with known outlier sparsity size and aLTS for adaptive estimation of outliers without knowing its precise number. In spite of the NP-hardness for finding global optimizers in the worst case, we show that such simple algorithms are able to reach statistically good estimates under mild conditions. Before the algorithms are described, a brief introduction on robust ranking is provided which motivates our main development.

\subsection{Robust Ranking}


Assume that there are $m$ raters and $n$ items to be ranked by the $m$ raters. Let $N$ be the the total number
of paired comparisons (samples). Let vector $\vY = (Y_{ij}^\alpha)_{i<j; \alpha} \in \mathbb{R}^N$ denote the
degree that rater $\alpha$ prefers item $i$ to item $j$. Without loss of generality, we assume that $Y_{ij}^\alpha>0$
if rater $\alpha$ prefers item $i$ to item $j$ and $Y_{ij}^{\alpha}< 0$ otherwise. In addition, we assume that the
paired comparison data is \emph{skew-symmetric} for each $\alpha$, i.e., $Y^\alpha_{ij}=-Y^\alpha_{ji}$.
In practice, $Y_{ij}^\alpha$ can be continuous, dichotomous or of a $k$-point Likert
scale with $k\geq 3$ according to the strategy used in QoE evaluation.

It is natural to assume that
\begin{equation}
Y_{ij}^\alpha = s_i^\ast - s_j^\ast + Z_{ij}^{\alpha *},
\end{equation}
$\vs^*=(s_1^*,\ldots,s_n^*)^T\in \R^{n}$ is the true ranking score on $n$ items and $Z_{ij}^{\alpha *}$ is the noise satisfying $Z_{ij}^{\alpha *} = - Z_{ji}^{\alpha *}$. When the noise $Z_{ij}^{\alpha *}$ is independent and identically distributed with zero mean, least squares (LS) problem has been used in~\cite{MM11,tmm12,MM12} to derive ranking scores in subjective multimedia assessments.

However, not all comparisons are trustworthy and there may be sparse outliers due to different test conditions, human errors, or abnormal variations in content. Putting in a mathematical way, here we consider
\begin{equation}\label{eq:linear}
Y_{ij}^\alpha = s_i^\ast - s_j^\ast + E_{ij}^{\alpha *} + N_{ij}^{\alpha *},
\end{equation}
or equivalently
\begin{equation}\label{eq:linear-Mtr}
\vY = \vX \vs^* + \vE^* + \vN^*.
\end{equation}
where $\vE^* = (E_{ij}^{\alpha *})\in \mathbb{R}^N$, which models the \emph{outliers}, is sparse and has a much larger magnitude than $\vN^* = (N_{ij}^{\alpha *})$, which models the Gaussian \emph{noise}, and $\vX\in \mathbb{R}^{N\times n}$ satisfies that: if $Y_{ij}^\alpha\ (i<j)$ is the $k$th entry of $\vY$, then the $k$th row of $\vX$ equals to $\mathrm{e}_i - \mathrm{e}_j$, here $\mathrm{e}_i\in \mathbb{R}^n$ satisfies that only the $i$th entry is $1$ and others are $0$. Such $\vX$ is often called the (generalized) ``gradient operator" on graph $G=(\{1,\ldots,n\},\ \{(i,j):Y^\alpha_{ij} \mbox{~exists}\})$, with $\mathbf{L} = \vX^T \vX$ being the (unnormalized) graph Laplacian.

When sparse outliers exist ($E_{ij}^{\alpha *}\neq 0$ for a small number of $(i,j,\alpha)$), the solution to the least squares problem on all the comparisons becomes unstable and may give an inaccurate estimation. If the outliers can be detected and removed, the solution to the least squares problem on the remaining pairwise comparisons is more accurate and gives a better estimation.

In~\cite{MM13}, a robust regression approach called Huber-LASSO is used to detect outliers:
\begin{equation} \label{eq:hlasso}
\Min\limits_{\vs\in {\mathbb{R}}^{n},\vE}  \frac{1}{2} \|\vY-\vX \vs - \vE\|_2^2+\lambda \|\vE\|_1.
\end{equation}
This is a convex optimization problem and the LASSO path $\lambda\mapsto \E_\lambda$ could provide information on the order of samples tending to be outliers.


However, there are two issues with this approach: 1) the Huber-LASSO estimator $\hat{\vs}$ is always biased, even under the identifiable condition $\vs \perp 1_n$ where $1_n=(1,\ldots,1)^T \in \mathbb{R}^n$; 2) computing the Huber-LASSO path to get top outliers is computationally expensive.

In order to remove the bias in the solution, we replace the $l_1$-norm of $\vE$ in~\eqref{eq:hlasso} with the $l_0$-``norm'' of $\vE$ and obtain
\begin{equation} \label{eq:iht}
\Min\limits_{\vs\in \Rn, \vE} \displaystyle\frac{1}{2} \| \vY - \vX \vs - \vE \|_2^2 + \lambda \|\vE\|_0.
\end{equation}
where $\|\vE\|_0$, the $l_0$-``norm'' of $\vE$, is the number of nonzero components in $\vE$. Although this is a nonconvex optimization problem which is NP-hard in the worst case, in the sequel we shall see that under mild conditions even simple iterative algorithms may detect where the outliers are and lead to statistically good estimators.


\subsection{iHT and iLTS with Known $K$}



\begin{algorithm}[htbp]
\caption{iterative Hard Thresholding (iHT)}	\label{alg:iht}
	\begin{algorithmic}
		\State \textbf{Input:} $\vY = (Y_{ij}^\alpha)$, $K\geq0$, $\epsilon > 0$.
		\State \textbf{Initialization:} $\vE^0 = (E_{ij}^\alpha)^0 = 0$.
			\For{$k=0,1,\ldots$}
				\State Update $\vE$ by
                    \begin{align*}
                        \vE^{k+1} = \Proj_K ((\vI_N - \vH)\vY + \vH \vE^k),
                    \end{align*}
                \State {\bf If} $\| \vE^{k+1} - \vE^k \| \le \epsilon$, break.
			\EndFor
			\State \Return $\hat \vE = \vE^k,\ \hat \vs = (\vX^T \vX)^\dag \vX^T (\vY - \vE^k)$.
	\end{algorithmic}
\end{algorithm}

First of all, Proposition~\ref{prop:iht-equiv}, whose proof is provided in the supplementary material, shows that problem~\eqref{eq:iht} is, in a sense, equivalent to
\begin{equation} \label{eq:iht2}
\left\{\begin{array}{rl}
\Min\limits_{\vs\in {\mathbb{R}}^{n},\vE}& \displaystyle\frac{1}{2} \| \vY - \vX \vs - \vE \|_2^2,\\
\textnormal{subject to }&\|\vE\|_0\leq K
\end{array}\right.
\end{equation}
and
\begin{equation} \label{eq:ho_rank_aop}
\left\{\begin{array}{rl}
\Min\limits_{\vs\in {\mathbb{R}}^{n},\Lambda}& \displaystyle\frac{1}{2} \| \Lambda \circ (\vY - \vX \vs) \|_2^2,\\
\textnormal{subject to }& \Lambda = (\Lambda_{ij}^\alpha) \in \{0,1\}^N,\ \|\Lambda\|_0\geq N - K
\end{array}\right.
\end{equation}
where $\circ$ is elementwise Hadamard product operator. The index of zero entries of $\Lambda$ indicate outliers. Problem~\eqref{eq:ho_rank_aop} is actually the Least Trimmed Squares (LTS) in robust regression \cite{leroy1987robust}. A benefit of \eqref{eq:ho_rank_aop} lies in that the global ranking score $\vs$ does not depend on the outlier magnitude estimate, by dropping off the outliers.

\begin{prop}\label{prop:iht-equiv}
For a given $\lambda > 0$, pick any global optimal $(\tilde\vs, \tilde\vE)$ for problem~\eqref{eq:iht}, and let $K = \|\tilde\vE\|_0$. Let
\begin{align*}
S_1 = \left\{\right. \vs:\ &\textnormal{$\|\vE\|_0 = K$ and}\\
& \textnormal{$(\vs,\vE)$ is optimal for problem~\eqref{eq:iht}}\left.\right\}\\
S_2 = \left\{\right. \vs:\ &\textnormal{$(\vs,\vE)$ is optimal for problem~\eqref{eq:iht2}}\left.\right\}\\
S_3 = \left\{\right. \vs:\ &\textnormal{$(\vs,\Lambda)$ is optimal for problem~\eqref{eq:ho_rank_aop}}\left.\right\}.
\end{align*}
Then $S_1 = S_2 = S_3$.
\end{prop}

Hence now we turn to problem~\eqref{eq:iht2}~and~\eqref{eq:ho_rank_aop}, both have a parameter $K$, which is considered as an upper bound of the number of outliers.  Because of the two $l_0$-``norm'', finding the global optimal solution is NP-hard. We attempt to find approximate (but sufficient) solutions via the alternating minimization method.

Note that once we fix $\vE = \vE^k$ for problem~\eqref{eq:iht2}, then we just need to solve an ordinary least squares problem and get a corresponding $\vs^k$ simply by
\begin{equation}\label{eq:iht-update-s}
\vs^k = (\vX^T \vX)^\dag \vX^T (\vY - \vE^k).
\end{equation}
Here $A^\dag$ is the Moore--Penrose pseudoinverse of a matrix $A$. And if we fix $\vs = \vs^k$, we just need to take a Hard Thresholding, i.e.
\begin{equation}\label{eq:iht-update-E}
\vE^{k+1} = \Proj_K \left( \vY - \vX \vs^k \right),
\end{equation}
where $\Proj_K$ is an operator which sets all entries to $0$ except $K$ entries with largest squares. For example,
\begin{align*}
\Proj_3 (-1,5,2,-4,-6) = (0,5,0,-4,-6).
\end{align*}
Plugging \eqref{eq:iht-update-s} into \eqref{eq:iht-update-E}, such a procedure implies
\begin{align*}
\vE^{k+1} & = \Proj_K(\vY - \vX (\vX^T \vX)^\dag \vX^T (\vY - \vE^k))\\
& = \Proj_K ((\vI_N - \vH)\vY + \vH \vE^k),
\end{align*}
where $\vH = \vX(\vX^T \vX)^\dag \vX^T$ is the ``hat matrix''. Such a procedure is described precisely in Algorithm~\ref{alg:iht} and called iterative Hard Thresholding (iHT).

For problem~\eqref{eq:ho_rank_aop}, when $\Lambda^k$ is fixed, update $\vs$ by solving a least squares problem using only the comparisons indicated by $\Lambda^k$, i.e.
\begin{equation}\label{eq:iLTS-update-s}
  \vs^k = (\vX^T \mathrm{diag}(\Lambda^k) \vX)^\dag (\vX^T \mathrm{diag}(\Lambda^k) \vY).
\end{equation}
When fixing $\vs = \vs^k$, updating $\Lambda$ is to choose $N - K$ entries of $\vY - \vX\vs^k$ with smallest squares, then set the $N - K$ corresponding entries of $\Lambda^{k+1}$ to be $1$, and others to be $0$. The procedure is described precisely in Algorithm~\ref{alg:ilts}.
\begin{algorithm}[t]
\caption{An Iterative Procedure for LTS (iLTS)}	\label{alg:ilts}
	\begin{algorithmic}
		\State \textbf{Input:} $\vY = (Y_{ij}^\alpha)$, $K\geq0$.
		\State \textbf{Initialization:} $\Lambda^0 = (\Lambda_{ij}^\alpha)^0=1_N$.
			\For{$k=0,1,\ldots$}
				\State Update $\vs$ to get $\vs^k$ by \eqref{eq:iLTS-update-s}.
                \State Update $\Lambda$ by choosing $N - K$ entries of $\vY - \vX\vs^k$ with smallest squares\footnotemark, then setting the $N - K$ corresponding entries of $\Lambda^{k+1}$ to be $1$, and others to be $0$. 
                \State Check if the new $\Lambda^{k+1}$ is different from all $\Lambda^l$ ($l\leq k$) appeared before. \textbf{If} not, break.
			\EndFor
			\State \Return $\hat \Lambda = \Lambda^k,\ \hat \vs = \vs^k$.
	\end{algorithmic}
\end{algorithm}

\footnotetext{If the $K$th and $(K+1)$th largest squares have the same value, there are multiple choices of $\Lambda^{k+1}$. In this case, randomly choose one of them different from all $\Lambda$'s appeared before. If all the choices have appeared, break.}

\subsection{Consistency of iHT and iLTS}
A natural question is, under what conditions can these two algorithms detect the true outlier set. The following theorems, whose proofs are given in the supplementary material, present some RIP-like sufficient conditions which can be met in outlier detection.
\begin{theorem}[Sparsistency of iHT]
\label{thm:iht-cons}
Assume that $\vY = (Y_{ij}^\alpha)$ satisfies the model \eqref{eq:linear-Mtr} with $\|\vE^*\|_0 = K^*$ and $\vE_{\min}^* = \min_{E_{ij}^{\alpha *}\neq 0} |E_{ij}^{\alpha *}|$.
Now, for arbitrary $K\ge K^*$ satisfying
\begin{equation} \label{eq:RIP1}
\theta := \sup_{J\subseteq \{1,2,\ldots,N\},\ |J|\le 3K} \left\| \vX_J (\vX^T \vX)^\dag \vX_J^T \right\|_2 < \frac{1}{2}
\end{equation}
(here $\vX_J$ is the submatrix consist of some columns of $\vX$ indexed by $J$), $\vE^k$ in Algorithm~\ref{alg:iht} converges to the true outlier vector $\vE^*$ in the following sense
\begin{align} \label{eq:iht-cons-E}
\|\vE^k - \vE^*\|_2\le (2\theta)^k \cdot \| \vE^0 - \vE^* \|_2 + \frac{2 \|\vN^*\|_2}{1-2\theta}.
\end{align}
Moreover, if
\begin{align} \label{eq:iht-cons-N}
\theta < \frac{1}{2} - \frac{\|\vN^*\|_2}{\vE_{\min}^*},
\end{align}
then for sufficiently large $k$, $\supp(\vE^k) \supseteq \supp(\vE^*)$ holds. If~\eqref{eq:iht-cons-N} holds and $K = K^*$ additionally, then for sufficiently large $k$, $\supp(\vE^k) = \supp(\vE^*)$ holds.
\end{theorem}

\begin{remark}
Condition~\eqref{eq:RIP1} resembles the condition in \cite{foucart2012sparse}, with the measurement matrix $A$ replaced by $\vI_N - \vH$, and the number of nonzero entries $s$ replaced by $K$.
\end{remark}

\begin{remark}
According to the statement of the theorem above, we should choose $K$ to be at least $K^*$. But it is unnecessary to exactly let $K$ be the unknown number $K^*$, since we allow $K$ to be larger than $K^*$. However, usually $K$ can not be too large, due to the condition $\theta < 1/2$ which must be satisfied.

In the definition of $\theta$, note that $\vX_J (\vX^T \vX)^\dag \vX_J^T := \vH_{J,J}$ is a $|J|\times |J|$ submatrix of $\vH$, and $\|\vH\|_2 \le 1$ always holds since $\vI_N - \vH$ is always positive semi-definite. If $3K$ (upper bound of $|J|$) is small enough, then $\theta$ can be smaller than $1/2$, satisfying the proposed condition. For example, if $n=10$, $K=1$, and each pair has exactly one comparison, then $\theta = 0.4 < 1/2$.
\end{remark}

\begin{theorem}[Convergence of iLTS]
\label{thm:lts-conv} Algorithm~\ref{alg:ilts} converges in finite steps. Moreover, let
\begin{align*}
F(\vs,\Lambda) =\ & \frac{1}{2} \| \Lambda \circ (\vY - \vX \vs) \|_2^2\\
& + \iota(\Lambda\in \{0,1\}^K,\ \|\Lambda\|_0\ge N - K),
\end{align*}
where $\iota(A)$ is the indicator function, which equals $0$ if the event $A$ happens, and equals $+\infty$ otherwise. Then the output $\vs^k$ with the corresponding $\Lambda^k$ satisfies
\begin{enumerate}
    \item
        $(\vs^k,\Lambda^k)$ is a coordinatewise minimum point of $F(\vs,\Lambda)$, namely, for any $\vs, \Lambda$,
        \begin{align*}
        F(\vs^k, \Lambda^k)\le F(\vs^k, \Lambda),\\
        F(\vs^k, \Lambda^k)\le F(\vs, \Lambda^k).
        \end{align*}
    \item
        $\vs^k$ is a local minimum point of $E(\vs) := \min_{\Lambda} F(\vs,\Lambda)$.
\end{enumerate}
\end{theorem}
\begin{remark}
There is no convergence analysis for iHT in general case. But this theorem tells that iLTS always converges, though they are two different iterative algorithms for two equivalent problems.
\end{remark}

\begin{theorem}[Sparsistency of iLTS]
\label{thm:lts-cons}
Assume that $\vY = (Y_{ij}^\alpha)$ satisfies the model \eqref{eq:linear-Mtr} with $\|\vE^*\|_0 = K^*,\ \vE_{\min}^* = \min_{E_{ij}^{\alpha *}\neq 0} |E_{ij}^{\alpha *}|$ and $\Lambda^*\in \{0,1\}^N$ satisfying
\begin{equation*}
\Lambda_{ij}^{\alpha *} =
\left\{\begin{array}{rl}
1, & E_{ij}^{\alpha *} = 0,\\
0, & E_{ij}^{\alpha *} \neq 0.
\end{array}\right.
\end{equation*}
Now, for arbitrary $K\ge K^*$, let
\begin{subequations}
\begin{align}
\mu &:= \sup_{|J|\ge N - K} \| \vX_{J^c} (\vX_J^T \vX_J)^\dag \vX_J^T \|_2\\
\eta &:= \sup_{|J|\ge N - K} \| \vX_{J^c} (\vI_n - (\vX_J^T \vX_J)^\dag (\vX_J^T \vX_J)) \|_2\\
\epsilon &:= \sqrt{2}\cdot \frac{(2+\mu)\|\vN^*\|_2 + \eta \|\vs^*\|_2}{\vE_{\min}^*}.
\end{align}
\end{subequations}
If
\begin{align} \label{eq:RIP2}
\varphi := \sup_{|J'|\le 2K,\ |J|\ge N - K} \| \vX_{J'} (\vX_J^T \vX_J)^\dag \vX_{J'}^T \|_2 < \sqrt{2} - 1 - \epsilon,
\end{align}
then for the $\hat\Lambda$ corresponding with the output $\hat \vs$ of Algorithm~\ref{alg:ilts}, $\supp(\hat\Lambda) \subseteq \supp(\Lambda^*)$ holds. If~\eqref{eq:RIP2} holds and $K = K^*$ additionally, then $\supp(\hat\Lambda) = \supp(\Lambda^*)$.
\end{theorem}

\begin{remark}
In the vast majority of cases, $\eta = 0$. In fact, as long as for each $J\subseteq \{1,2,\ldots,N\},\ |J|\ge N-K$, any row of $\vX_{J^c}$ is a linear combination of $\vX_J$ (which means that, removing the samples indicated by rows of $\vX_{J^c}$ does not disturb the original structure of connected components of the graph), there is a matrix $\mathbf{M}$ such that
\begin{align*}
\vX_{J^c} = \mathbf{M} \vX_J.
\end{align*}
Thus
\begin{align*}
& \vX_{J^c} (\vI_n - (\vX_J^T \vX_J)^\dag (\vX_J^T \vX_J))\\
=\ & \mathbf{M} (\vX_J - \vX_J \cdot (\vX_J^T \vX_J)^\dag \vX_J^T \cdot \vX_J)\\
=\ & \mathbf{M} (\vX_J - \vX_J \vX_J^\dag \vX_J) = 0,
\end{align*}
which implies that $\eta = 0$.
\end{remark}

\begin{remark}
According to the statement of the theorem above, we should choose $K$ to be at least $K^*$. But it is unnecessary to exactly let $K$ be the unknown number $K^*$, since we allow $K$ to be larger than $K^*$. However, usually $K$ can not be too large, due to the condition $\varphi < \sqrt{2} - 1 - \epsilon$ which must be satisfied.
\end{remark}

\begin{remark}
Conditions \eqref{eq:RIP1} and \eqref{eq:RIP2} play similar roles as Restricted Isometry Property (RIP) in compressed sensing \cite{CanTao05}.
\end{remark}


\subsection{Adaptive LTS with Unknown $K$} \label{subsec:aLTS}
If the exact number of outliers $K$ is given or can be accurately estimated, Algorithm~\ref{alg:iht} or~\ref{alg:ilts} can be used to detect the outliers and improve the performance of least squares solutions. However, in practice, the exact number of outliers $K$ is generally unknown. If $K$ is underestimated, we are able to remove some outliers and the remaining outliers will still damage the performance of the least squares solutions. On the other hand, if $K$ is overestimated, too many comparisons are removed. The resulting data is not enough for robust QoE evaluation and provides unstable solutions. Therefore, a method to estimate the number of outliers accurately is strongly desired.

We propose a method to estimate the number of outliers automatically for dichotomous choice $Y_{ij}^\alpha\in \{\pm 1\}$. In this case, a natural way is to consider those outliers as the paired comparisons which disagree with the sign (or preference order) of global ranking score differences.


As the number of outliers is unknown, firstly we use the least squares problem to find an estimation of $\vs$, then the total number of comparisons with wrong directions ($Y_{ij}^\alpha$ has different sign with $s_i-s_j$), which is denoted as $\widetilde{K}$, is an overestimation of $K$. Then we obtain an underestimation of the number of outliers via multiplying by $\beta_1\in(0,1)$, i.e., $\utilde{K}=\beta_1\widetilde{K}$. We remove $\utilde{K}$ comparisons that have largest violations to the current score because they are most likely to be outliers. The remaining comparisons are used to find the new estimation of $\vs$ via the least squares problem. In this case, we are able to remove some outliers and improve the estimation for $\vs$. With these improved estimation for $\vs$, we are able to remove more outliers. So we increase the underestimation $\utilde{K}$ by $\beta_2$ ($\beta_2\in (1,\infty)$). However, this number can not be larger than $\widetilde{K}$, the smallest overestimation of the number of outliers, because we do not want to remove too many comparisons.
Therefore the update of $\utilde{K}$ is just $\utilde{K}=\min(\lceil\beta_2\utilde{K}\rceil,\widetilde{K})$ where $\lceil x \rceil$ is the smallest integer no smaller than positive real number $x$. Iterations go on until $\utilde{K}=\widetilde{K}$, and it gives an accurate estimation of the number of outliers. This algorithm is named aLTS for adaptive Least Trimmed Squares, and Algorithm \ref{alg:IRLS} describes such a procedure precisely.

\begin{algorithm}[ht]
\caption{adaptive LTS (aLTS)}	
\label{alg:IRLS}
	\begin{algorithmic}
		\State \textbf{Input:} $\vY = (Y_{ij}^\alpha)$, $\beta_1<1$, $\beta_2>1$.
		\State \textbf{Initialization:} $\Lambda^0 = 1_N$, $\utilde{K}^{-1}=0$, $\widetilde{K}^{-1}=+\infty$.
		\For{$k=0,1,\ldots$}
                    \State Update $\vs$ to get $\vs^k$ by \eqref{eq:iLTS-update-s}.
			        \State Let $\widetilde{K}^k$ be the total number of comparisons with wrong directions, i.e., $Y_{ij}^\alpha$ has different sign with $s^k_i-s^k_j$.
			        \begin{align*}
							\widetilde{K}^k  &=\min \{\widetilde{K}^k,\widetilde{K}^{k-1}\}.\\
\utilde{K}^k&=\left\{
\begin{array}{ll}
\lceil\beta_1 \widetilde{K}^{k}\rceil,& \mbox{if $k=0$}; \\
\min(\lceil \beta_2 \utilde{K}^{k-1}\rceil, \widetilde{K}^k),& \mbox{otherwise}.
\end{array}
\right.
\end{align*}
			        \State {\bf If} $\utilde{K}^k=\widetilde{K}^k$, break.
		          	\State Update $\Lambda$ to get $\Lambda^{k+1}$ in the same way as in Algorithm~\ref{alg:ilts}, with $K$ replaced by $\utilde{K}^k$.
		\EndFor			
			\State \Return $\hat \Lambda = \Lambda^k,\ \hat\vs = \vs^k,\ \hat K=\widetilde{K}^k$.
	\end{algorithmic}
\end{algorithm}

\begin{remark} \label{remark2}There are only two parameters to choose, and these two parameters are easy to set. They are chosen according to inequalities $\beta_1<1 < \beta_2$ ($\beta_1=0.75$ and $\beta_2=1.03$ are fixed in our numerical experiments). $\beta_1$ has to be small to make sure that the first estimation of the number of outliers is underestimated. Then the underestimation $\utilde{K}$ increases geometrically with rate $\beta_2$, and $\beta_2$ can not be too large, because the remain comparisons are not enough for robust QoE evaluation after too many comparisons are removed. \end{remark}

\begin{remark} \label{remark3}The algorithm is able to detect most of the outliers in our experiments. However, there may be mistakes in the detection, and these mistakes happen mostly between two successive items in the order. Therefore, we can add one step to just compare every pair of two successive items and make the correction on the detection, i.e., if $s_i^k>s_j^k$, but $|\{Y_{ij}^\alpha:Y_{ij}^\alpha>0\}| < |\{Y_{ij}^\alpha:Y_{ij}^\alpha<0\}|$, then remove $\{Y_{ij}^\alpha:Y_{ij}^\alpha<0\}$ from outliers and add in $\{Y_{ij}^\alpha:Y_{ij}^\alpha>0\}$.
\end{remark}

Algorithm~\ref{alg:IRLS} always stops in finite steps, as shown in the following lemma.
\begin{lemma}
Algorithm \ref{alg:IRLS} stops in no more than $k^*$ steps, where
\[ k^* =\left\lfloor \frac{ -\log \beta_1 }{ \log \beta_2} \right\rfloor + 2. \]
\end{lemma}
\begin{proof}
It follows from the fact that the sequence $\{\widetilde{K}^k\}$ is non-increasing, and $\{\utilde{K}^k\}$ is a geometrically increasing sequence which is bounded by the smallest component of $\{\widetilde{K}^k\}$. Specifically, assume that $k^*$ steps have been taken in Algorithm~\ref{alg:IRLS}, then $k$ has approached $k^*-1$, and $\utilde{K}^{k} \ge \beta_2 \utilde{K}^{k-1}$ for $0<k<k^*-1$, so
\begin{align*}
\widetilde{K}^0 \ge \widetilde{K}^{k^*-2} \ge \utilde{K}^{k^*-2} \ge \beta_2^{k^*-2} \utilde{K}^0 \ge \beta_2^{k^*-2} \beta_1 \widetilde{K}^0,
\end{align*}
which leads to the result.
\end{proof}

Such a result only ensures that the algorithm stops with a possible overestimation of the number of outliers because $\widetilde{K}^k$ is always an overestimation for the number of outliers. The following theorem presents a stability condition when Algorithm \ref{alg:IRLS} returns the correct number of outliers.

\begin{theorem} \label{thm:alts} Consider binary choice data with outliers
\begin{equation}\label{outlier-def-2}
Y_{ij}^\alpha \mbox{~is~an~outlier,~if~} Y_{ij}^\alpha \neq \sign(s_i^\ast - s_j^\ast).
\end{equation}
Assume that there exists an integer $k_0$ such that for all $k\geq k_0$, least squares estimator $\vs^k$ is order-consistent to the true score $\vs^*$, i.e., $\vs^k$ induces the same ranking order as the true score ${\vs}^\ast$, then Algorithm \ref{alg:IRLS} returns the correct number of outliers.
\end{theorem}
\begin{proof}
As $\vs^k$ is an order-consistent solution of the ground-truth, by definition, $\widetilde{K}^k$ gives the correct number of outliers, say $K^\ast$. It actually holds for all $k\geq k_0$, that $\widetilde{K}^k= K^\ast$. From Lemma 1, the claim follows.
\end{proof}

\begin{remark}
One scenario is the generalized linear model where $p(i\succeq j) = f(s_i^\ast - s_j^\ast)$ for some cumulate distribution function $f$ symmetric w.r.t. $f(0)=1/2$. With a large enough sample, all the pairwise preferences in the minority direction can be regarded as ``outliers" and dropping such outliers will not change the order consistency of least square estimators.
\end{remark}

Note that Theorem \ref{thm:alts} does not require $\Lambda^k$ to correctly identify the outliers, but just stable estimator $\vs^k$ to be order-consistent to $\vs^\ast$. In practice, this might not be satisfied easily. But, as we shall see in the next section, Algorithm \ref{alg:IRLS} typically returns stable estimators that only deviate locally.

\begin{figure}[t]
\setlength{\abovecaptionskip}{0pt}
 \begin{center}
 \subfigure [SN=1000]{
\includegraphics[width=0.31\linewidth]{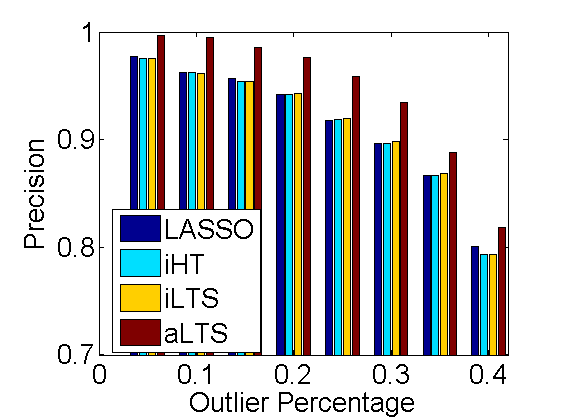}}
 \subfigure[SN=2000]{
\includegraphics[width=0.31\linewidth]{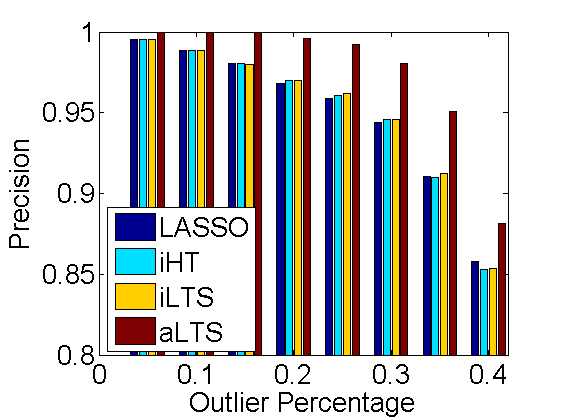}}
 \subfigure[SN=3000]{
\includegraphics[width=0.31\linewidth]{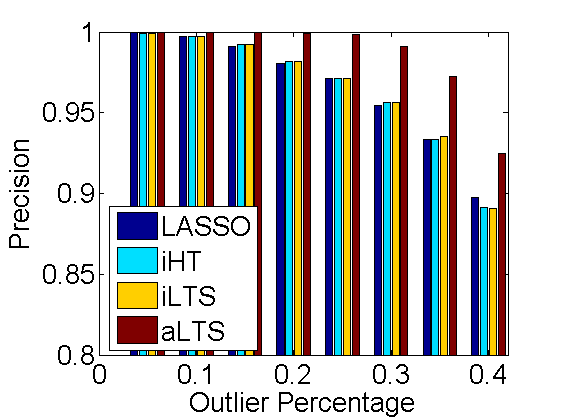}}
\renewcommand{\captionfont}{\scriptsize \bfseries}
\caption{\textbf{\emph{Precision}}s for simulated data via LASSO, iHT, iLTS, and aLTS, 100 times repeat.} \label{fig:precision}
\end{center}
\end{figure}

\begin{figure}[t]
\setlength{\abovecaptionskip}{0pt}
 \begin{center}
 \subfigure [SN=1000]{
\includegraphics[width=0.31\linewidth]{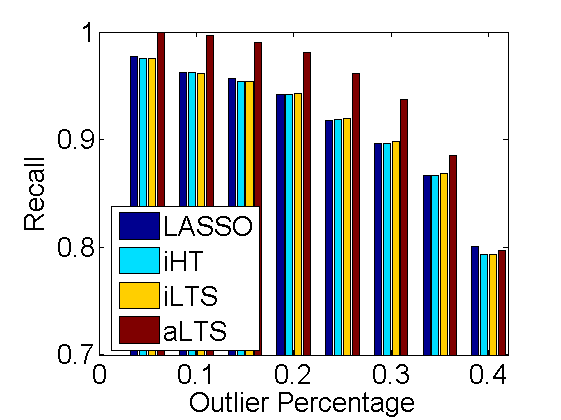}}
 \subfigure[SN=2000]{
\includegraphics[width=0.31\linewidth]{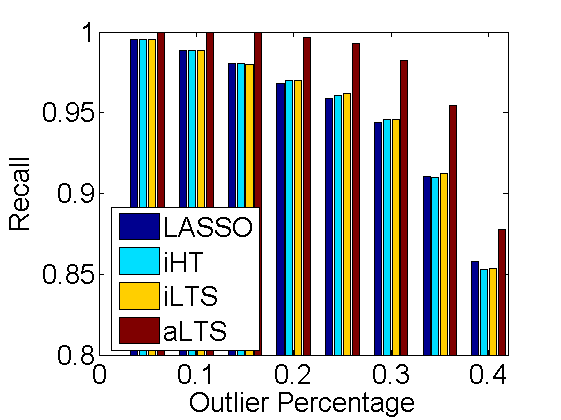}}
 \subfigure[SN=3000]{
\includegraphics[width=0.31\linewidth]{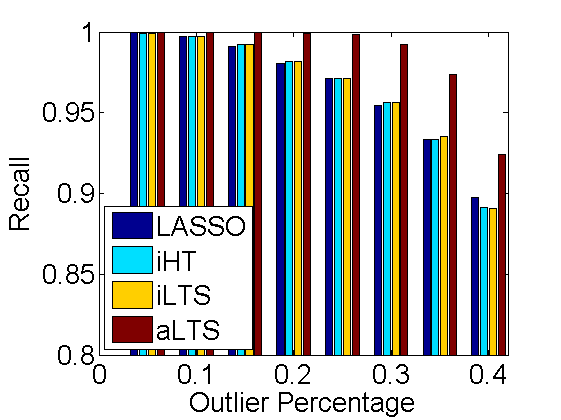}}
\renewcommand{\captionfont}{\scriptsize \bfseries}
\caption{\textbf{\emph{Recall}}s for simulated data via LASSO, iHT, iLTS, and aLTS, 100 times repeat.} \label{fig:recall}
\end{center}
\end{figure}

\section{EXPERIMENTS}\label{sec:experiments}

A key question in the outlier detection community is how to evaluate the effectiveness of outlier detection algorithms when the ground-truth outliers are not available. In this section, we will first show the effectiveness of the proposed method on simulated data with known ground-truth outliers, followed by real-world crowdsourcing datasets without ground-truth outliers.

\subsection{Simulated Data}
The simulated data is constructed as follows. A random total order on $n$ items is created as the ground-truth order. Then we add paired comparison edges $(i,j)$ randomly with preference directions following the ground-truth order. We simulate the outliers by randomly choosing a portion of the comparison edges and reversing them in preference directions. A paired comparison graph with outliers, possibly incomplete and imbalanced, is constructed.

Here we choose $n=16$, which is consistent with the real-world datasets, and make the following definitions for the experimental parameters. The total number of paired comparisons occurred on the graph is {\bf SN} (Sample Number), and the number of outliers is {\bf ON} (Outlier Number). Then the outlier percentage {\bf OP} can be obtained as {\bf ON}/{\bf SN}.

Most outlier detection algorithms adopt a tuning parameter (say \emph{t}) in order to select different number of data samples as outliers~\cite{MM13}, and the number of outliers detected changes as $t$ changes. If \emph{t} is picked too restrictively, then the algorithm will miss true outlier (false negatives). On the other hand, if the algorithm declares too many data samples as outliers, then it will lead to too many false positives. This tradeoff can be measured in terms of \textbf{\emph{precision}} and \textbf{\emph{recall}}, which are commonly used for measuring the effectiveness of outlier detection methods. Specifically, the \textbf{\emph{precision}} is defined as the percentage of reported outliers that truly turn out to be outliers; and the \textbf{\emph{recall}} is correspondingly defined as the percentage of ground-truth outliers that have been reported as outliers.

\begin{figure}[t]
\setlength{\abovecaptionskip}{0pt}
 \begin{center}
 \subfigure [SN=1000]{
\includegraphics[width=0.31\linewidth]{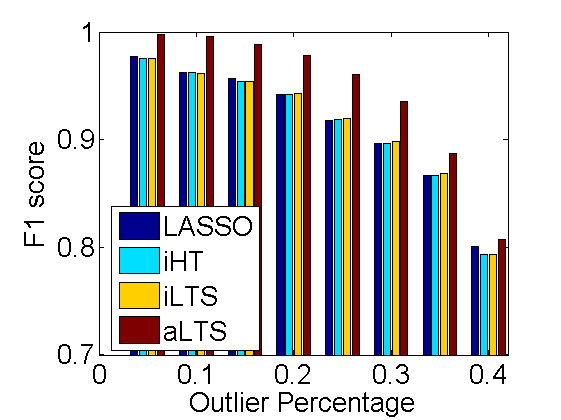}}
 \subfigure[SN=2000]{
\includegraphics[width=0.31\linewidth]{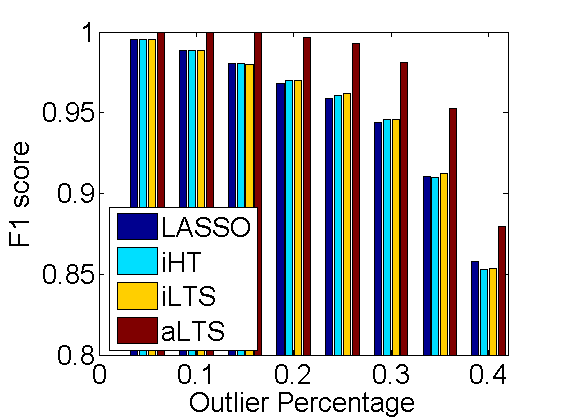}}
 \subfigure[SN=3000]{
\includegraphics[width=0.31\linewidth]{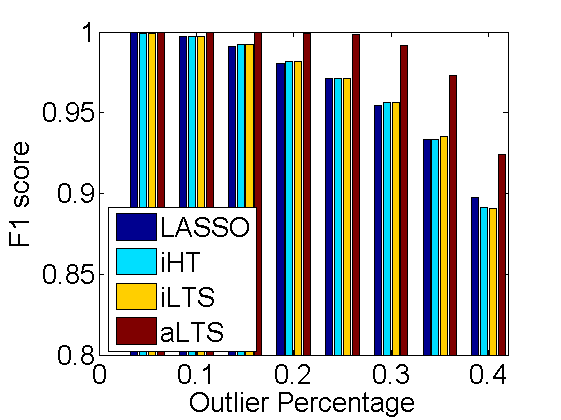}}
\renewcommand{\captionfont}{\scriptsize \bfseries}
\caption{\textbf{\emph{F1}} scores for simulated data via LASSO, iHT, iLTS, and aLTS, 100 times repeat.} \label{fig:f1}
\end{center}
\end{figure}

We then compare LASSO, iHT, iLTS, and aLTS for outlier detection on the simulated data. For ease of comparison, here we should tell LASSO, iHT, and iLTS in advance the exact number of outliers existed in the dataset. Because, different from aLTS, these three methods can not estimate the number of outliers in the dataset automatically.

The mean \textbf{\emph{precision}}s, \textbf{\emph{recall}}s, and \textbf{\emph{F1-score}}s over 100 runs for these four methods on different choices of {\bf SN} and {\bf OP} are shown in Figures~\ref{fig:precision},~\ref{fig:recall}, and ~\ref{fig:f1}. \textbf{\emph{F1-score}} is a combined measure that assesses the \textbf{\emph{precision}}/\textbf{\emph{recall}} tradeoff, which reaches its best value at 1 and worst score at 0.

It is easy to see that the performances of LASSO, iHT, and iLTS are very similar, while aLTS could produce better performance (indicated by higher \textbf{\emph{precision}}s, \textbf{\emph{recall}}s, and \textbf{\emph{F1-score}}s in almost all cases).
In addition, we compare the computing time required for these four methods to finish all the 100 runs in
Tables~\ref{runtime}. All computation is done using MATLAB R2014a, on a Mac Pro desktop PC, with 2.8 GHz Intel Core i7-4558u, and 16 GB memory.
It is easy to see that on the simulated dataset, iHT, iLTS, and aLTS algorithms are much faster than LASSO, which implies their advantages in dealing with large-scale data.
Specifically, iHT and iLTS can achieve up to about 30--90 times faster than LASSO, and aLTS is almost 3--8 times faster than the time for LASSO. As aLTS does not have any information about the number of outliers existed in the dataset and should estimate the number of outliers automatically, its computation cost is reasonably more expensive compared with iHT and iLTS.

\subsection{Real-world Data}\label{sec:VQA}

%
%

Two crowdsourcing real-world datasets are adopted in this subsection. Since there is no ground-truth for outliers in real-world datasets, we can not compute~\textbf{\emph{precision}} and~\textbf{\emph{recall}} as in the simulated data to evaluate the performance of the methods. Therefore, we inspect the outliers returned by four methods and compare them with the whole data to see whether they are reasonably good outliers or not.


The first dataset PC-VQA, which is collected by~\cite{MM11}, contains 38,400 pairwise comparisons of the LIVE dataset~\cite{LIVE} from 209 random raters. The paired comparison data in this dataset is complete and balanced. Take reference (a) in the PC-VQA dataset as an illustrative example (other reference videos exhibit similar results). The number of outliers estimated by aLTS is used for LASSO/iHT/iLTS to choose the regularization parameter and select the outliers.

{\renewcommand\baselinestretch{1.1}\selectfont

\begin{table} [t] \renewcommand{\captionfont}{\scriptsize \bfseries}
\caption{\label{runtime} Computing time for 100 runs in total on simulated data via LASSO, iHT, iLTS,  and aLTS.}
\tiny
\centering

\subtable[LASSO]{
\newsavebox{\tablebox}
\begin{lrbox}{\tablebox}
    \begin{tabular}{c||p{0.7cm}p{0.7cm}p{0.7cm}p{0.7cm}p{0.7cm}p{0.7cm}p{0.7cm}p{0.7cm}}
 \hline  time (s)   &\textbf{OP=5\%}  &\textbf{OP=10\%} &\textbf{OP=15\%} &\textbf{OP=20\%} &\textbf{OP=25\%} &\textbf{OP=30\%} &\textbf{OP=35\%} &\textbf{OP=40\%} \\
 \hline
 \hline  \textbf{SN=1000}    &18.62   &19.72   &22.51   &23.80   &23.48   &22.56   &21.05   &18.56 \\
 \hline  \textbf{SN=2000}    &20.58   &29.21   &33.17   &34.81   &34.57   &31.54   &29.82   &25.78	  \\
  \hline  \textbf{SN=3000}   &28.59   &37.50   &40.62   &40.88   &41.60   &38.91   &34.94   &29.38	 \\
 \hline
 \end {tabular}
 \end{lrbox}
\scalebox{0.8}{\usebox{\tablebox}}
       \label{tab:lasso}
}

\subtable[iHT]{
\begin{lrbox}{\tablebox}
      \begin{tabular}{c||p{0.7cm}p{0.7cm}p{0.7cm}p{0.7cm}p{0.7cm}p{0.7cm}p{0.7cm}p{0.7cm}}
 \hline  time (s)   &\textbf{OP=5\%}  &\textbf{OP=10\%} &\textbf{OP=15\%} &\textbf{OP=20\%} &\textbf{OP=25\%} &\textbf{OP=30\%} &\textbf{OP=35\%} &\textbf{OP=40\%} \\
 \hline
 \hline  \textbf{SN=1000}     &0.23    &0.20    &0.24    &0.27    &0.31    &0.35    &0.40    &0.43 \\
 \hline  \textbf{SN=2000}    &0.29    &0.33    &0.40    &0.48    &0.50    &0.57    &0.65    &0.72	  \\
  \hline  \textbf{SN=3000}  &0.41    &0.48    &0.55    &0.60    &0.69    &0.79    &0.83    &0.97  \\
 \hline
 \end {tabular}
  \end{lrbox}
\scalebox{0.8}{\usebox{\tablebox}}
       \label{tab:iHT}
}

\subtable[iLTS]{
\begin{lrbox}{\tablebox}
     \begin{tabular}{c||p{0.7cm}p{0.7cm}p{0.7cm}p{0.7cm}p{0.7cm}p{0.7cm}p{0.7cm}p{0.7cm}}
 \hline  time (s)   &\textbf{OP=5\%}  &\textbf{OP=10\%} &\textbf{OP=15\%} &\textbf{OP=20\%} &\textbf{OP=25\%} &\textbf{OP=30\%} &\textbf{OP=35\%} &\textbf{OP=40\%} \\
 \hline
 \hline  \textbf{SN=1000}   &0.27    &0.24    &0.30    &0.33    &0.37    &0.40    &0.43    &0.45 \\
 \hline  \textbf{SN=2000}  &0.37    &0.41    &0.53    &0.60    &0.63    &0.73    &0.79    &0.83  \\
  \hline  \textbf{SN=3000}  &0.53    &0.63    &0.70    &0.78    &0.90    &1.03    &0.99    &1.12	  \\
 \hline
 \end {tabular}
  \end{lrbox}
\scalebox{0.8}{\usebox{\tablebox}}
       \label{tab:iHT}
}

\subtable[aLTS]{
\begin{lrbox}{\tablebox}
     \begin{tabular}{c||p{0.7cm}p{0.7cm}p{0.7cm}p{0.7cm}p{0.7cm}p{0.7cm}p{0.7cm}p{0.7cm}}
 \hline  time (s)   &\textbf{OP=5\%}  &\textbf{OP=10\%} &\textbf{OP=15\%} &\textbf{OP=20\%} &\textbf{OP=25\%} &\textbf{OP=30\%} &\textbf{OP=35\%} &\textbf{OP=40\%} \\
 \hline
 \hline  \textbf{SN=1000}    &4.86    &3.50    &3.13    &2.79    &3.00    &2.93    &2.85   &2.81 \\
 \hline  \textbf{SN=2000}    &6.36    &5.35    &4.97    &4.91    &4.75    &4.42    &4.29    &4.27	  \\
  \hline  \textbf{SN=3000}  &7.96    &7.61    &6.81    &6.70    &6.34    &6.09    &5.51    &5.92	  \\
 \hline
 \end {tabular}
  \end{lrbox}
\scalebox{0.8}{\usebox{\tablebox}}
       \label{tab:iHT}
}

\end{table}
\par}

{\renewcommand\baselinestretch{1.1}\selectfont
\setlength{\belowcaptionskip}{5pt}
\renewcommand{\captionfont}{\scriptsize \bfseries}
\begin{table}[ht]
\caption{\label{matrixs1} Paired comparison matrices of reference (a) in PC-VQA dataset. Red numbers are overlapping outliers obtained by LASSO, iHT, iLTS, and aLTS. Open blue circles are those obtained by LASSO/iHT/iLTS but not aLTS, while filled blue circles are those obtained by aLTS but not LASSO/iHT/iLTS.}
\scriptsize
\centering
\begin{lrbox}{\tablebox}
\begin{tabular}{|c|c|c|c|c|c|c|c|c|c|c|c|c|c|c|c|c|}
\hline Video ID    &\textbf{1} &\textbf{9} &\textbf{10} &\textbf{13} &\textbf{7} &\textbf{8} &\textbf{11} &\textbf{14} &\textbf{15} &\textbf{3} &\textbf{12} &\textbf{4} &\textbf{16} &\textbf{5} &\textbf{6} &\textbf{2} \\
\hline \textbf{1}    & 0   &  22    & 29    & 30    & 30    & 29    & 29    & 29    & 30    & 28    & 29    & 32    & 32    & 31   &  32   &  31 \\
\hline \textbf{9}     &\red{10}    &  0   & 22    & 20    & 14    & 23    & 23    & 25    & 29    & 29    & 32    & 30    & 29    & 30    & 29    & 31 \\
\hline \textbf{10}    &\red{3}    &\red{10}     & 0    & 22    & 11    & 21   &  29    & 23    & 31    & 27    & 31   &  30   &  32    & 30    & 32    & 31 \\
\hline \textbf{13}    &\red{2}     &\red{12}    &\red{10}     & 0    & 18   &  22   &  23    & 27   &  31   &  28    & 29   &  29    & 29    & 25    & 27    & 28 \\
\hline \textbf{7}    &\red{2}    &\red{18}   &\red{21}   &\red{14}    &  0   &  21    & 14    & 16   &  28    & 23    & 31    & 25   &  19   &  27    & 26    & 28 \\
\hline \textbf{8}   &\red{3}     &\red{9}   &\red{11}   &\red{10}    &\red{11}     & 0    & 25   &  14    & 28   &  25    & 29   &  27    & 24   &  25    & 28    & 32  \\
\hline \textbf{11}    &\red{3}     &\red{9}      &\red{3}     &\red{9}    &\red{18}     &\red{7}     & 0    & 22    & 27    & 26    & 26    & 30   &  30   &  27   &  27    & 31  \\
\hline \textbf{14}   &\red{3}     &\red{7}       &\red{9}      &\red{5}    &\red{16}   &\red{18}    &\red{10}    &  0    & 28    & 27    & 18   &  29   &  29   &  26    & 28   &  29 \\
\hline \textbf{15}    &\red{2}     &\red{3}     &\red{1}    &\red{1}     &\red{4}      &\red{4}      &\red{5}     &\red{4}     &  0    & 25   &  20    & 22    & 26    & 25   &  29    & 24 \\
\hline \textbf{3}    &\red{4}     &\red{3}    &\red{5}    &\red{4}     &\red{9}    &\red{7}      &\red{6}    &\red{5}      &\red{7}    &  0    &\red{11}   &\textcolor{white}{\pgftextcircledblk{\bf 15}}    & 26    & 24    & 29    & 28  \\
\hline \textbf{12}   &\red{3}    &  0    &\red{1}    &\red{3}     &\red{1}     &\red{3}     &\red{6}    &\red{14}     &\red{12}   &  21     & 0    & 16    & 20    & 24    & 26    & 26 \\
\hline \textbf{4}   & 0     &\red{2}     &\red{2}     &\red{3}     &\red{7}     &\red{5}     &\red{2}     &\red{3}     &\red{10}    & \textcolor{blue}{\pgftextcircled{\bf 17}}    &\red{16}    &  0   &  15    & 26    & 27    & 30 \\
\hline \textbf{16}   & 0     &\red{3}     & 0     &\red{3}    &\red{13}     &\red{8}     &\red{2}    &\red{3}     &\red{6}     &\red{6}   &\red{12}    &\red{17}     & 0   &  22    & 24    & 28 \\
\hline \textbf{5}  &\red{1}    &\red{2}    &\red{2}     &\red{7}    &\red{5}     &\red{7}     &\red{5}   &\red{6}    &\red{7}     &\red{8}    &\red{8}    &\red{6}    &\red{10}    & 0   &  26    & 27  \\
\hline \textbf{6}   & 0     &\red{3}    &  0     &\red{5}     &\red{6}     &\red{4}     &\red{5}   &\red{4}    &\red{3}     &\red{3}     &\red{6}      &\red{5}    &\red{8}   &\red{6}    &  0   &  21  \\
\hline \textbf{2}    &\red{1}    &\red{1}    &\red{1}    &\red{4}     &\red{4}     & 0     &\red{1}     &\red{3}     &\red{8}     &\red{4}    &\red{6}      &\red{2}      &\red{4}     &\red{5}    &\red{11}    & 0 \\
\hline
\end {tabular}
  \end{lrbox}
\scalebox{0.65}{\usebox{\tablebox}}
\end{table}
 \par}

Outliers detected by these methods are shown in the paired comparison matrix in Table~\ref{matrixs1}. The paired comparison matrix is constructed as follows (Table~\ref{matrixs2} is constructed in the same way). For each video pair $\{i,j\}$, let $n_{ij}$ be the number of comparisons for items $i$ and $j$, among which $a_{ij}$ raters agree that the quality of item $i$ is better than item $j$ ($a_{ji}$ carries the opposite meaning). So $a_{ij} + a_{ji} = n_{ij}$ if no tie occurs. In the PC-VQA dataset, $n_{ij}=32$ for any video pair $\{i,j\}$. The order of the video IDs in the matrix is arranged such that the global ranking score calculated by the least squares problem with all the comparisons is decreasing (from high to low). The number of outliers estimated by aLTS from this reference video is 716. So we choose the parameter for LASSO/iHT/iLTS to detect 716 outliers, and the exact number of outliers returned by LASSO/iHT/iLTS is 718, which is slightly larger than 716.

The outliers detected by these methods are mainly distributed in the lower left corner of this matrix, which implies that the outliers are those comparisons with large deviations from the global ranking scores by LS. It is easy to see that outliers returned by LASSO, iHT, iLTS, and aLTS are almost the same except on one pair (ID = 3 and ID = 4). In this dataset, 15 raters agree that the quality of ID = 3 is better than that of ID = 4, while 17 raters have the opposite opinion. LASSO, iHT, and iLTS return the same results which tend to choose comparisons with large deviations from the global ranking scores as outliers. So these three treat the 17 comparisons preferring ID = 4 as outliers because ID = 3 ranks above ID = 4. However, aLTS prefers to choose the minority as outliers and treats the 15 comparisons preferring ID = 3 as outliers. Such a small difference only leads to a local order change of ID = 3 and ID = 4. Therefore the ranking algorithms are stable.

{\renewcommand\baselinestretch{1.1}\selectfont
\setlength{\belowcaptionskip}{5pt}
\renewcommand{\captionfont}{\scriptsize \bfseries}
\begin{table}[t]
\caption{\label{matrixs2} Paired comparison matrices of reference (c) in PC-IQA dataset. Red numbers, open blue circles, and filled blue circles carry the same meanings as in Table~\ref{matrixs1}.}
\scriptsize
\centering
\begin{lrbox}{\tablebox}
\begin{tabular}{|c|c|c|c|c|c|c|c|c|c|c|c|c|c|c|c|c|}
\hline \textbf{Image ID}      &\textbf{1} &\textbf{8} &\textbf{16} &\textbf{2} &\textbf{3} &\textbf{11} &\textbf{6} &\textbf{12} &\textbf{9} &\textbf{14} &\textbf{5} &\textbf{13} &\textbf{7} &\textbf{10} &\textbf{15 }&\textbf{4} \\
\hline \textbf{1}    &0	&13	&9	&16	&19	&12	&15	&13	&14	&14	&14	&17	&16	&17	 &16	&16 \\
\hline \textbf{8}     &\red{6}	&0	&8	&7	&8	&5	&13	&7	&7	&8	&19	&8	 &15	&9	&12	&15 \\
\hline \textbf{16}    &\red{4}	&0	&0	&9	&11	&9	&8	&15	&3	&18	&16	&17	 &12	&7	&21	&18\\
\hline \textbf{2}    &\red{5}	&\red{5}	&\red{6}	&0	&8	&9	&10	&11	 &7	&14	 &13	&14	&14	&13	&14	 &15 \\
\hline \textbf{3}     &\red{3}	&\red{4}	&\red{6}	&\red{7}	&0	&6	 &11	&9	&10	&16	&12	 &15	&14	&14	&18	 &13 \\
\hline \textbf{11}    &\red{4}	&\red{6}	&\red{3}	&\red{5}	 &\red{6}	&0	 &\textcolor{white}{\pgftextcircledblk{\bf 5}}	&3	 &5	&6	 &21	 &5	 &11	 &7	 &12	 &18\\
\hline \textbf{6}   &0	&\red{2}	&\red{7}	&\red{4}	&\red{2}	 &\textcolor{blue}{\pgftextcircled{\bf 7}}	&0	 &12	&12	&7	 &22	&15	 &17	&13	&13	&17   \\
\hline \textbf{12}   &\red{3}	&\red{4}	&\red{1}	&\red{4}	 &\red{4}	 &\red{3}	&\red{1}	&0	 &8	 &15	 &18	&12	&9	 &8	&13	 &17 \\
\hline \textbf{9}   &\red{1}	&\red{3}	&\red{3}	&\red{5}	 &\red{1}	 &\red{3}	&\red{1}	&0	 &0	 &5	 &18	 &10	&14	&9	 &7	 &16  \\
\hline \textbf{14}   &0	&0	&\red{1}	&0	&0	&\red{3}	&\red{7}	 &\red{2}	 &\red{1}	&0	&14	&15	&10	&8	&17	 &19 \\
\hline \textbf{5}    &0	&0	&0	&0	&0	&0	&0	&0	&0	&\red{1}	&0	&14	 &19	&19	&15	&17 \\
\hline \textbf{13}   &0	&0	&0	&0	&0	&0	&0	&0	&0	&0	&\red{6}	&0	 &5	&7	&17	&16  \\
\hline \textbf{7}    &0	&0	&0	&0	&0	&0	&0	&0	&0	&0	&0	&\red{5}	 &0	&8	&9	&18 \\
\hline \textbf{10}    &0	&0	&0	&0	&0	&0	&0	&0	&0	&0	&0	 &\red{2}	&\red{2}	&0	 &\textcolor{white}{\pgftextcircledblk{\bf 3}}	 &11  \\
\hline \textbf{15}   &0	&0	&0	&0	&0	&0	&0	&0	&0	&0	&0	&0	&0	 &\textcolor{blue}{\pgftextcircled{\bf 5}}	&0	 &11 \\
\hline \textbf{4}  &0	&0	&0	&0	&0	&0	&0	&0	&0	&0	&0	&\red{1}	 &0	&\red{6}	 &\red{6}	&0\\\hline
\end {tabular}
  \end{lrbox}
\scalebox{0.71}{\usebox{\tablebox}}
\end{table}
 \par}

The global ranking scores of the four algorithms, namely LASSO, iHT, iLTS, and aLTS are shown in Table~\ref{tab:data5-rank}. For the ease of seeing the differences on global rating scores after outlier detection, we also report the results obtained by LS which has been used in~\cite{MM11,tmm12} to derive ranking scores in subjective multimedia assessments. After the detected outliers are removed, the orders of some competitive videos are changed. LASSO, iHT, iLTS, and aLTS all think that ID = 12 has better performance than ID = 3 and ID = 4. However, the orders of ID = 3 and ID = 4 are exchanged in aLTS and LASSO/iHT/iLTS, because they choose different preference directions as outliers.


The second dataset PC-IQA~\cite {MM12} is incomplete and imbalanced. This dataset contains 15 reference images and 15 distorted versions of each reference image. So the total number of images is 240. These images come from two publicly available datasets: LIVE~\cite{LIVE} and IVC~\cite{IVC}. Totally, 186 raters, each of whom performs a varied number of comparisons via Internet, provide 23,097 pairwise comparisons.

Tables~\ref{matrixs2} and~\ref{tab:data10-rank} show the comparable experimental results of LASSO, iHT, iLTS, and aLTS on reference image (c) (other reference images exhibit similar results). The number of outliers estimated by aLTS is 173, so we choose the parameter of LASSO/iHT/iLTS to detect 173 outliers. The exact number of outliers returned by LASSO/iHT/iLTS is 177, which is slightly larger than 173. We can see that the difference of the detection between LASSO/iHT/iLTS and aLTS happens on two pairs: 1) ID = 6 and ID = 11; 2) ID = 10 and ID = 15. Same as in the last experiment, these methods differ in outlier detection for highly comparable pairs. aLTS prefers to choose the minority in paired comparisons, i.e., the 5 comparisons preferring ID = 11 over ID = 6 and the 3 comparisons preferring ID = 10 over ID = 15, while LASSO/iHT/iLTS selects comparisons with largest deviations from global ranking scores even when the votings are in majority. Such a difference leads to a local order change of involved items only.

{\renewcommand\baselinestretch{1.33}\selectfont

\begin{table} [t]
\renewcommand{\captionfont}{\scriptsize \bfseries}
\caption{\label{matrixs} Comparison of different rankings. Five ranking methods are compared with the integer representing the ranking position and the number in parentheses representing the global ranking score returned by the corresponding algorithm.}

\makeatletter\def\@captype{table}\makeatother
\begin{minipage}[t]{0.23\textwidth}
\scriptsize
\centering

\subtable[Ref (a) in the PC-VQA]{

\begin{lrbox}{\tablebox}
      \begin{tabular}{||c|c|c|c||}

      \hline\hline
    \textbf{ID} & \textbf{LS} & \textbf{LASSO/iHT/iLTS}  & \textbf{aLTS}\\
    \hline
 1 & 1 ( 0.7930 ) & 1 ( 0.9123 )  & 1 ( 0.9129 ) \\
9  & 2 ( 0.5312 ) & 2 ( 0.7537 )  & 2 ( 0.7539 ) \\
10 & 3 ( 0.4805 ) & 3 ( 0.6317 )  & 3 ( 0.6322 )  \\
13 & 4 ( 0.3906 ) & 4 ( 0.5522 )  & 4 ( 0.5524 )   \\
7 & 5 ( 0.2852 ) & 5 ( 0.4533 )  & 5 ( 0.4537 )   \\
8 & 6 ( 0.2383 )  & 6 ( 0.3159 )   & 6 ( 0.3163 )  \\
11 & 7 ( 0.2148 ) &  7 ( 0.2113 )  & 7 ( 0.2120 )   \\
14 & 8 (  0.1641 ) & 8  ( 0.1099 )  & 8 ( 0.1103 )  \\
15 & 9 ( -0.1758 ) & 9 ( -0.1024 ) & 9 ( -0.1029 )   \\
3 & 10 ( -0.2227 ) & \red{11 ( -0.3195 )} & \red{12 ( -0.3999 )}  \\
12 & 11 ( -0.2500 ) & \red{10 ( -0.2149 )}  & \red{10 ( -0.2158 )}   \\
4 & 12 ( -0.2930 ) & \red{12 ( -0.4054 )}  & \red{11 ( -0.3252 )}   \\
16 & 13 ( -0.3633 ) & 13 ( -0.5311 )  & 13 ( -0.5332 )   \\
5 & 14 ( -0.4414 ) & 14 ( -0.6573 )  & 14 ( -0.6568 )   \\
6 & 15 ( -0.6289 ) & 15 ( -0.8054 )  & 15 ( -0.8057 )   \\
2 & 16 ( -0.7227 ) & 16 ( -0.9046 )  & 16 ( -0.9042 )   \\
 \hline
    \hline
\end{tabular}
\end{lrbox}
\scalebox{0.55}{\usebox{\tablebox}}
       \label{tab:data5-rank}
}
\end{minipage}
\makeatletter\def\@captype{table}\makeatother
\begin{minipage}[t]{0.23\textwidth}
\scriptsize
\centering
\subtable[Ref (c) in the PC-IQA]{
\begin{lrbox}{\tablebox}
       \begin{tabular}{||c|c|c|c||}
     \hline
    \hline
    \textbf{ID} & \textbf{LS} & \textbf{LASSO/iHT/iLTS}  & \textbf{aLTS}\\
    \hline
 1 & 1 ( 0.7575 ) & 1 ( 0.9015 ) & 1 ( 0.9022 ) \\
8  & 2 ( 0.5670 ) & 2 ( 0.7088 )  & 2 ( 0.7129 ) \\
16 & 3 ( 0.5124 ) & 3 ( 0.6472 )  & 3 ( 0.6504 )  \\
2 & 4 ( 0.4642 ) & 4 ( 0.5242 )  & 4 ( 0.5248 )   \\
3 & 5 ( 0.4423 ) & 5 ( 0.4119 )  & 5 ( 0.4148 )   \\
11 & 6 ( 0.3277 ) &  6 ( 0.2592 )  &  \red{ 7 ( 0.1763 )}  \\
6 & 7 ( 0.3128 ) & 7 ( 0.2515 )  &  \red{ 6 ( 0.3124 )}   \\
12 & 8 ( 0.2423 ) & 8 ( 0.1209 )  & 8 ( 0.1261 )  \\
9 & 9 ( 0.1453 ) & 9 ( 0.0043 )  & 9 ( 0.0069 )   \\
14 & 10 ( -0.0455 ) & 10 ( -0.1274 )  & 10 ( -0.1243 )   \\
5 & 11 ( -0.3376 ) & 11 ( -0.3205 )  & 11 ( -0.3214 )   \\
13 & 12 ( -0.4785 ) & 12 ( -0.4621 )  & 12 ( -0.4560 )   \\
7 & 13 ( -0.5396 ) &  13 ( -0.5515 )  & 13 ( -0.5494 )  \\
10 & 14 ( -0.7486 ) &  14 ( -0.7005 )  & \red{ 15 ( -0.7485 )}   \\
15 & 15 ( -0.7658 ) & 15 ( -0.7511 )  & \red{ 14 ( -0.7106 )}    \\
4 & 16 ( -0.8559 ) & 16 ( -0.9163 )  & 16 ( -0.9166 )   \\
 \hline
    \hline

\end{tabular}
\end{lrbox}
\scalebox{0.55}{\usebox{\tablebox}}
\label{tab:data10-rank}

}
\end{minipage}
\end{table}
\par}

\subsection{Discussion}
As we have seen in the numerical experiments, LASSO, iHT, iLTS, and aLTS mostly find the same outliers, and when they disagree, aLTS tends to choose the minority and LASSO/iHT/iLTS prefer to choose comparisons with large deviations from the global ranking scores even when the votings are in majority. When outliers consist of minority voting as in simulated experiments, aLTS performs better than LASSO, iHT, and iLTS. This can also be explained from the algorithm. We choose a small underestimation for the number of outliers, and increase this estimation until there is no outliers in the remaining comparisons. The parameter $\beta_2>1$ is chosen to be small so we will not overestimate the number of outliers too much.

%

Finally, we would like to point out that subject-based outlier detection can be a straightforward extension from our proposed algorithms. From the detection results, one may evaluate the reliability of one rater based on all the comparisons from the rater and remove all the comparison from unreliable raters.

\section{CONCLUSIONS}\label{sec:conclusions}
In this paper, we proposed fast algorithms for outlier detection with nonconvex optimization and robust ranking in QoE evaluation.
Specifically, for known $K$, the proposed iHT and iLTS could provide us almost the same performance
compared with LASSO, and the computational speed can achieve up to 90 times faster than LASSO. For unknown $K$,
we proposed an adaptive method called aLTS which could estimate the number of outliers and detect them
without any prior information about the number of outliers in the dataset. This method is nearly 3--8
times faster than LASSO. The effectiveness and efficiency of the proposed methods is demonstrated on both simulated
examples and real-world applications. The small distinctions between these four methods indicate that aLTS prefers
to choosing minority voting data as outliers, while the LASSO, iHT, and iLTS select the comparisons with largest deviations
from the global ranking score as outliers even when they are in majority. In both cases, the global rankings obtained are
stable. In summary, we expect that the proposed outlier detection methods for QoE will be helpful tools for people in the multimedia community exploiting crowdsourceable paired comparison data for robust ranking.

%

\bibliographystyle{abbrv}

\bibliography{sigproc}  

\newpage
\appendix
\section{Proofs}

\begin{proof}[(Proposition~\ref{prop:iht-equiv})]
First we prove $S_1 = S_2$. For any $\vs \in S_1$, there is $\vE$ such that $\|\vE\|_0 = K$ and $(\vs, \vE)$ is optimal for problem~\eqref{eq:iht}. Then for any $(\vs', \vE')$ such that $\|\vE'\|_0\le K$, since
\begin{align*}
\frac{1}{2} \|\vY - \vX \vs - \vE\|_2^2 + \lambda \|\vE\|_0 \le \frac{1}{2} \|\vY - \vX \vs' - \vE'\|_2^2 + \lambda \|\vE'\|_0
\end{align*}
and $\|\vE'\|_0 \le K = \|\vE\|_0$, we have
\begin{align*}
\frac{1}{2} \|\vY - \vX \vs - \vE\|_2^2 \le \frac{1}{2} \|\vY - \vX \vs' - \vE'\|_2^2.
\end{align*}
Hence $(\vs,\vE)$ is optimal for problem~\eqref{eq:iht2}, i.e. $\vs \in S_2$.

For any $\vs \in S_2$, there is $\vE$ such that $\|\vE\|_0 \le K$ and $(\vs, \vE)$ is optimal for problem~\eqref{eq:iht2}. Then for the pre-chosen $(\tilde\vs, \tilde\vE)$, since
\begin{align*}
\frac{1}{2} \|\vY - \vX \vs - \vE\|_2^2 \le \frac{1}{2} \|\vY - \vX \tilde \vs - \tilde \vE\|_2^2
\end{align*}
and $\|\vE\|_0 \le K = \|\tilde\vE\|_0$, we sum them up to get
\begin{align*}
\frac{1}{2} \|\vY - \vX \vs - \vE\|_2^2 + \lambda \|\vE\|_0 \le \frac{1}{2} \|\vY - \vX \tilde \vs - \tilde \vE\|_2^2 + \lambda \|\tilde \vE\|_0.
\end{align*}
Note that $(\tilde\vs, \tilde\vE)$ is optimal for problem~\eqref{eq:iht}, hence equality must hold, i.e.
\begin{align*}
\frac{1}{2} \|\vY - \vX \vs - \vE\|_2^2 &= \frac{1}{2} \|\vY - \vX \tilde \vs - \tilde \vE\|_2^2,\\
\|\vE\|_0 & = \|\tilde\vE\|_0 = K.
\end{align*}
Hence $(\vs,\vE)$ is optimal for problem~\eqref{eq:iht} as well as $\|\vE\|_0 = K$, i.e. $\vs \in S_1$. Thus $S_1 = S_2$.

Then we prove $S_2 = S_3$. For any $\vs\in S_2$, there is $\vE$ such that $\|\vE\|_0 \le K$ and $(\vs, \vE)$ is optimal for problem~\eqref{eq:iht2}. Since $\vE$ is optimal for $\vs$, it is easy to know that the index of nonzero entries of $\vE$ is contained in $J$ which is the index of $K$ entries of $\vY - \vX \vs$ with largest squares. Let $\Lambda\in \{0,1\}^N$ satisfying $\Lambda_J = 1_{N-K}$ and $\Lambda_{J^c} = 0_K$. For any $(\vs', \Lambda')$ such that $\|\Lambda'\|_0\ge N - K$, let $\vE' = (1_N - \Lambda')\circ (\vY - \vX \vs')$, then $\|\vE'\|_0 \le K$ and hence
\begin{align*}
& \frac{1}{2} \|\Lambda' \circ (\vY - \vX \vs')\|_2^2 = \frac{1}{2} \|\vY - \vX \vs' - \vE' \|_2^2\\
\ge\ & \frac{1}{2} \|\vY - \vX \vs - \vE\|_2^2 = \frac{1}{2} \|\Lambda \circ (\vY - \vX \vs)\|_2^2.
\end{align*}
Hence $(\vs,\Lambda)$ is optimal for problem~\eqref{eq:ho_rank_aop}, i.e. $\vs \in S_3$.

For any $\vs\in S_3$, there is $\Lambda$ such that $\|\Lambda\|_0 \ge N - K$ and $(\vs, \Lambda)$ is optimal for problem~\eqref{eq:ho_rank_aop}. For any $(\vs', \Lambda')$ such that $\|\Lambda'\|_0 \ge N - K$, let $\vE' = (1_N - \Lambda')\circ (\vY - \vX \vs')$ and $\vE = (1_N - \Lambda)\circ (\vY - \vX \vs)$, then $\|\vE'\|_0, \|\vE\|_0\le K$ and
\begin{align*}
& \frac{1}{2} \|\vY - \vX \vs' - \vE' \|_2^2 = \frac{1}{2} \|\Lambda' \circ (\vY - \vX \vs')\|_2^2\\
\ge\ & \frac{1}{2} \|\Lambda \circ (\vY - \vX \vs)\|_2^2 \ge \frac{1}{2} \|\vY - \vX \vs - \vE\|_2^2.
\end{align*}
Hence $(\vs,\vE)$ is optimal for problem~\eqref{eq:iht2}, i.e. $\vs \in S_2$. Thus $S_2 = S_3$.
\end{proof}

\begin{proof}[(Theorem~\ref{thm:iht-cons})]
Note that
\begin{align*}
\vE^{k+1} &= \Proj_K ((\vI_N - \vH)\vY + \vH \vE^k)\\
& = \mathrm{argmin}_{\|\vE\|_0\le K} \| (\vI_N - \vH)\vY + \vH \vE^k - \vE \|_2^2,
\end{align*}
and $\|\vE^*\|_0 = K^* \le K$, we have
\begin{align*}
& \|(\vI_N - \vH)\vY + \vH\vE^k - \vE^*\|_2^2\\
\ge\ & \|(\vI_N - \vH)\vY + \vH \vE^k - \vE^{k+1}\|_2^2\\
=\ & \|((\vI_N - \vH)\vY + \vH\vE^k - \vE^*) + (\vE^* - \vE^{k+1})\|_2^2.
\end{align*}
Expanding the right hand side and simple calculations imply
\begin{align*}
\|\vE^{k+1} - \vE^*\|_2^2\le 2 ((\vI_N - \vH)\vY + \vH\vE^k - \vE^*)^T (\vE^{k+1} - \vE^*).
\end{align*}
Plug $\vY = \vX\vs^* + \vN^* + \vE^*$ in and note that $(\vI_N - \vH) \vX = 0$,
the above right hand side becomes
\begin{align*}
2(\vE^k - \vE^*)^T \vH (\vE^{k+1} - \vE^*) + 2\vN^{*T}(\vI_N - \vH)(\vE^{k+1} - \vE^*).
\end{align*}
Let $J_k = \supp(\vE^*)\cup\supp(\vE^k)\cup\supp(\vE^{k+1})$, then $|J_k|\le 3K$ and
\begin{align*}
& (\vE^k - \vE^*)^T \vH (\vE^{k+1} - \vE^*)\\
=\ & (\vE_{J_k}^k - \vE_{J_k}^*)^T \vH_{J_k, J_k} (\vE_{J_k}^{k+1} - \vE_{J_k}^*)\\
\le\ & \theta \cdot \|\vE^k - \vE^*\|_2 \cdot \|\vE^{k+1} - \vE^*\|_2.
\end{align*}
Besides, since $\vI_N - (\vI_N - \vH)$ is positive semi-definite, we know $\|\vI_N - \vH\|_2\le 1$ and $\| (\vI_N - \vH) \vN^* \|_2 \le \|\vN^*\|_2$. Combining the above analysis, we obtain
\begin{align*}
& \|\vE^{k+1} - \vE^*\|_2 \le 2\theta \cdot \|\vE^k - \vE^*\|_2 + 2\|\vN^*\|_2,
\end{align*}
from which we can prove~\eqref{eq:iht-cons-E} by induction.

Moreover, if~\eqref{eq:iht-cons-N} holds, according to~\eqref{eq:iht-cons-E} we know that for sufficiently large $k$,
\begin{align*}
\|\vE^k - \vE^*\|_{\infty} \le \|\vE^k - \vE^*\|_2 < \vE_{\min}^*,
\end{align*}
which implies $\supp(\vE^k) \supseteq \supp(\vE^*)$. When $K = K^*$ additionally, due to the fact that $\|\vE^k\|_0 \le K = K^* = \|\vE^*\|_0$, we have $\supp(\vE^k) = \supp(\vE^*)$.
\end{proof}

\begin{proof}[(Theorem~\ref{thm:lts-conv})]
For any $\vs$, Let
\begin{align*}
\tau(\vs) = (\tau_1(\vs), \ldots, \tau_N(\vs))^T = (\vY - \vX \vs)\circ (\vY - \vX \vs),
\end{align*}
and $\tau_{(N-K+1)}(\vs)$ be the $(N-K+1)$th smallest ($K$th largest) value of entries of $\tau(\vs)$. $\tau_1(\vs), \ldots, \tau_N(\vs)$ and $\tau_{(N-K+1)}(\vs)$ are continuous functions of $\vs$, thus we can find a sufficiently small $\epsilon > 0$ such that for any $\vs$ satisfying $\|\vs - \vs^k \|_2 < \epsilon$,
\begin{equation*}
\left\{
\begin{array}{rl}
\tau_l(\vs) < \tau_{(N-K+1)}(\vs),& \textnormal{if}\ \tau_l(\vs^k) < \tau_{(N-K+1)}(\vs^k),\\
\tau_l(\vs) > \tau_{(N-K+1)}(\vs),& \textnormal{if}\ \tau_l(\vs^k) > \tau_{(N-K+1)}(\vs^k).
\end{array}
\right.
\end{equation*}
Now, for any given $\vs$ satisfying $\|\vs - \vs^k\|_2 < \epsilon$, we can find an optimal $\bar\Lambda$ for $\vs$, so that $\|\bar\Lambda\|_0 = N - K$. Such a $\bar\Lambda$ satisfies
\begin{equation*}
\left\{
\begin{array}{rl}
\bar\Lambda_l = 1,& \tau_l(\vs) < \tau_{(N-K+1)}(\vs),\\
\bar\Lambda_l = 0,& \tau_l(\vs) > \tau_{(N-K+1)}(\vs).
\end{array}
\right.
\end{equation*}
Hence
\begin{equation*}
\left\{
\begin{array}{rl}
\bar\Lambda_l = 1,& \tau_l(\vs^k) < \tau_{(N-K+1)}(\vs^k),\\
\bar\Lambda_l = 0,& \tau_l(\vs^k) > \tau_{(N-K+1)}(\vs^k),
\end{array}
\right.
\end{equation*}
which implies that $\bar\Lambda$ is optimal not only for $\vs$, but also for $\vs^k$. Because the algorithm stops at $\vs^k$, we know $\bar\Lambda$ must have appeared before, i.e. there is $j\le k$ such that $\bar\Lambda = \Lambda^j$, and
\begin{align*}
F(\vs^k, \Lambda^j) \le F(\vs^k, \Lambda^k).
\end{align*}
Note that $\vs^j$ is optimal for $\Lambda^j$, from the definition of the algorithm. Thus
\begin{align*}
F(\vs^k, \Lambda^j) \ge F(\vs^j, \Lambda^j).
\end{align*}
Because $F$ is non-increasing during the procedure of the algorithm, we have $F(\vs^j, \Lambda^j) \ge F(\vs^k, \Lambda^k)$. Combining it with the above inequalities, we obtain
\begin{align*}
F(\vs^j, \Lambda^j) = F(\vs^k, \Lambda^j) = F(\vs^k, \Lambda^k).
\end{align*}
So not only $\vs^j$ but also $\vs^k$ is optimal for $\Lambda^j$, we have then
\begin{align*}
E(\vs^k) = F(\vs^k, \Lambda^j) \le F(\vs, \Lambda^j) = F(\vs, \bar\Lambda) = E(\vs).
\end{align*}
Therefore $\vs^k$ is a local minimum point of $E(\vs)$.

Finally, in the above analysis, the equality $F(\vs^k, \Lambda^j) = F(\vs^k, \Lambda^k)$ tells us not only $\Lambda^j$ but also $\Lambda^k$ is optimal for $\vs^k$. Besides, $\vs^k$ is obviously optimal for $\Lambda^k$, from the definition of the algorithm. So $(\vs^k, \Lambda^k)$ is a coordinatewise minimum point of $F(\vs,\Lambda)$.
\end{proof}

\begin{proof}[(Theorem~\ref{thm:lts-cons})]
From Theorem~\ref{thm:lts-conv}, we know that Algorithm~\ref{alg:ilts} finally converges in finite steps. Assume that output is $\vs$, with corresponding $\Lambda\in \{0,1\}^N$. Call $J = \{l:\ \Lambda_l = 1\}$ the ``index set of $\Lambda$'', and $J^*$ the index set of $\Lambda^*$ similarly. Obviously $\vE_{J^*}^* = 0$. Define
\begin{align*}
J_1 = J^*\backslash J,\ J_2 = J\backslash J^*,\ J' = J_1 \cup J_2.
\end{align*}
Since $\vs$ is optimal for $\Lambda$,
\begin{align*}
\vs & = (\vX_J^T \vX_J)^\dag \vX_J^T \vY_J = \vX_J^\dag \vY_J\\
& = \vs^* - \Delta \vs^* + \vX_J^\dag \vE_J^* + \vX_J^\dag \vN_J^*,
\end{align*}
where $\Delta = I_n - \vX_J^\dag \vX_J$. Note that $\vX_J \Delta = 0$, we have
\begin{equation}\label{XJ-delta-s}
\vX_J (\vs - \vs^*) = \vX_J \vX_J^\dag \vE_J^* + \vX_J \vX_J^\dag \vN_J^*.
\end{equation}
Since $\Lambda$ is optimal for $\vs$,
\begin{align*}
\| \Lambda \circ (\vY - \vX \vs) \|_2 \le \| \Lambda^* \circ (\vY - \vX \vs) \|_2.
\end{align*}
Plug in $\vY = \vX\vs^* + \vN^* + \vE^*$ and get
\begin{align*}
\| \vE_J^* + \vN_J^* + \vX_J(\vs^* - \vs)\|_2 \le \| \vE_{J^*}^* + \vN_{J^*}^* + \vX_{J^*}(\vs^* - \vs) \|_2.
\end{align*}
Plugging \eqref{XJ-delta-s} into the left hand side gives us
\begin{align*}
& \| (\vI_{N-K} - \vX_J \vX_J^\dag) (\vE_J^* + \vN_J^*) \|_2\\
\ge\ & \| (\vI_{N-K} - \vX_J \vX_J^\dag) \vE_J^* \|_2 - \| (\vI_{N-K} - \vX_J \vX_J^\dag) \vN_J^* \|_2\\
=\ & \| (\vI_{N-K} - \vX_J \vX_J^\dag) \vE_J^* \|_2\\
& - \sqrt{\vN_J^{*T}(\vI_{N-K} - \vX_J (\vX_J^T \vX_J)^\dag \vX_J^T) \vN_J^*}\\
\ge\ & \| (\vI_{N-K} - \vX_J \vX_J^\dag) \vE_J^* \|_2 - \|\vN^*\|_2,
\end{align*}
while the right hand side is
\begin{align*}
& \| \vE_{J^*}^* + \vN_{J^*}^* + \vX_{J^*}(\vs^* - \vs) \|_2\\
=\ & \| 0 + \vN_{J^*}^* + \vX_{J^*} (\Delta \vs^* - \vX_J^\dag \vE_J^* - \vX_J^\dag \vN_J^*) \|_2\\
\le\ & \| \vN_{J^*}^* - \vX_{J^*} \vX_J^\dag \vN_J^* \|_2 + \| \vX_{J^*} \Delta \vs^* \|_2 + \| \vX_{J^*} \vX_J^\dag \vE_J^* \|_2\\
\le\ & \| \vN_J^* - \vX_J \vX_J^\dag \vN_J^* \|_2 + \| \vN_{J_1}^* - \vX_{J_1} \vX_J^\dag \vN_J^* \|_2\\
& + \| \vX_J \Delta \vs^* \|_2 + \| \vX_{J_1} \Delta \vs^* \|_2 + \| \vX_{J^*} \vX_J^\dag \vE_J^* \|_2\\
\le\ & \| (\vI_{N-K} - \vX_J (\vX_J^T \vX_J)^\dag \vX_J^T) \vN_J^* \|_2\\
& + \|\vN_{J_1}^*\|_2 + \|\vX_{J_1} (\vX_J^T\vX_J^T)^\dag \vX_J^T \vN_J^*\|_2\\
& + 0 + \| \vX_{J_1} \Delta \vs^* \|_2 + \| \vX_{J^*} (\vX_J^T\vX_J)^\dag \vX_J^T \vE_J^* \|_2\\
\le\ & \|\vN_J^*\|_2 + \|\vN_{J_1}^*\|_2 + \mu \|\vN_J^*\|_2 + \eta \|\vs^*\|_2\\
& + \| \vX_{J^*} (\vX_J^T\vX_J)^\dag \vX_J^T \vE_J^* \|_2\\
\le\ & (1 + \mu) \|\vN^*\|_2 + \eta \|\vs^*\|_2 + \| \vX_{J^*} (\vX_J^T\vX_J)^\dag \vX_J^T \vE_J^* \|_2.
\end{align*}
Hence
\begin{align*}
& \| (\vI_{N-K} - \vX_J \vX_J^\dag) \vE_J^* \|_2 - \| \vX_{J^*} (\vX_J^T\vX_J)^\dag \vX_J^T \vE_J^* \|_2\\
\le & (2 + \mu) \|\vN^*\|_2 + \eta \|\vs^*\|_2 = \frac{\sqrt{2}}{2} \epsilon \cdot \vE_{\min}^*.
\end{align*}
The first and second term of the left hand side above are denoted by $A,B$ respectively, then
\begin{align*}
A - B \le \frac{\sqrt{2}}{2} \epsilon \cdot \vE_{\min}^*.
\end{align*}
According to our assumption,
\begin{align*}
\vI_{|J'|} + \vX_{J'} (\vX_J^T \vX_J)^\dag \vX_{J'}^T & \prec (\sqrt{2} - \epsilon)\cdot \vI_{|J'|}\\
& \prec \sqrt{2 - \sqrt{2}\epsilon}\cdot \vI_{|J'|},
\end{align*}
where $U\succ V\Longleftrightarrow V\prec U$ means that $U - V$ is a positive definite matrix. Squaring on both sides lead to
\begin{align*}
& (1 - \sqrt{2}\epsilon)\cdot \vI_{|J'|} - 2 \vX_{J'} (\vX_J^T \vX_J)^\dag \vX_{J'}^T\\
& - \vX_{J'} (\vX_J^T \vX_J)^\dag \vX_{J'}^T \vX_{J'} (\vX_J^T \vX_J)^\dag \vX_{J'}^T \succ 0,
\end{align*}
so its submatrix
\begin{align*}
\mathbf{M} =\ & (1 - \sqrt{2}\epsilon)\cdot \vI_{|J_2|} - 2 \vX_{J_2} (\vX_J^T \vX_J)^\dag \vX_{J_2}^T\\
& - \vX_{J_2} (\vX_J^T \vX_J)^\dag \vX_{J'}^T \vX_{J'} (\vX_J^T \vX_J)^\dag \vX_{J_2}^T
\end{align*}
is also positive definite. If $\supp(\Lambda)\not\subseteq \supp(\Lambda^*)$, i.e. $J\not\subseteq J^*$, then $J_2 = J\backslash J^*$ is not empty, and $\vE_{J_2}^*\neq 0$. Since $J^*\subseteq J\cup J'$ implies $\vX_J^T \vX_J + \vX_{J'}^T \vX_{J'}\succeq \vX_{J^*}^T \vX_{J^*}$, and note that $\vE_{J\cap J^*}^* = 0$, we have
\begin{align*}
& A^2 - B^2\\
=\ & \vE_J^{*T} (\vI_{N-K} - \vX_J (\vX_J^T \vX_J)^\dag \vX_J^T) \vE_J^*\\
& - \vE_J^{*T} (\vX_J (\vX_J^T \vX_J)^\dag (\vX_{J^*}^T \vX_{J^*}) (\vX_J^T \vX_J)^\dag \vX_J^T) \vE_J^*\\
\ge\ &\vE_{J_2}^{*T} (\vI_{|J_2|} - \vX_{J_2} (\vX_J^T \vX_J)^\dag \vX_{J_2}^T\\
& - \vX_{J_2} (\vX_J^T \vX_J)^\dag (\vX_J^T \vX_J + \vX_{J'}^T \vX_{J'}) (\vX_J^T \vX_J)^\dag \vX_{J_2}^T) \vE_{J_2}^*\\
=\ & \vE_{J_2}^{*T} (\mathbf{M} + \sqrt{2}\epsilon \cdot \vI_{|J_2|}) \vE_{J_2}^* > \sqrt{2}\epsilon\cdot \|\vE_{J_2}^*\|_2^2 > 0.
\end{align*}
So $A>B$ and
\begin{align*}
A + B = \frac{A^2 - B^2}{A - B} > \frac{\sqrt{2}\epsilon\cdot \|\vE_{J_2}^*\|_2^2}{\frac{\sqrt{2}}{2}\epsilon \cdot \vE_{\min}^*} \ge 2\|\vE_{J_2}^*\|_2.
\end{align*}
However, from the definition of $A$ and the fact that $A>B$, we have
\begin{align*}
A + B & < 2A\\
& = 2\sqrt{\vE_{J_2}^{*T} (\vI_{|J_2|} - \vX_{J_2} (\vX_J^T \vX_J)^\dag \vX_{J_2}^T) \vE_{J_2}^*} \le 2 \|\vE_{J_2}^*\|_2,
\end{align*}
a contradiction. Hence $\supp(\Lambda)\subseteq \supp(\Lambda^*)$.

When $K = K^*$ additionally, due to the fact that $\|\Lambda^k\|_0 \le N - K = N - K^* = \|\Lambda^*\|_0$, we have $\supp(\Lambda^k) = \supp(\Lambda^*)$.
\end{proof}

\end{document}